\documentclass{article}

\usepackage{arxiv}

\usepackage[utf8]{inputenc} 
\usepackage[T1]{fontenc}    
\usepackage{hyperref}       
\usepackage{url}            
\usepackage{booktabs}       
\usepackage{amsfonts}       
\usepackage{nicefrac}       
\usepackage{microtype}      
\usepackage{lipsum}
\usepackage{graphicx}
\graphicspath{ {./images/} }

\usepackage{amsmath,amsfonts,bm}

\def\eqref#1{equation~\ref{#1}}

\def\1{\bm{1}}

\DeclareMathAlphabet{\mathsfit}{\encodingdefault}{\sfdefault}{m}{sl}
\SetMathAlphabet{\mathsfit}{bold}{\encodingdefault}{\sfdefault}{bx}{n}

\usepackage{amsmath,amsfonts,bm}

\def\eqref#1{equation~\ref{#1}}

\def\1{\bm{1}}

\DeclareMathAlphabet{\mathsfit}{\encodingdefault}{\sfdefault}{m}{sl}
\SetMathAlphabet{\mathsfit}{bold}{\encodingdefault}{\sfdefault}{bx}{n}

\usepackage{algorithm}
\usepackage{algorithmic}
\usepackage{amsthm}

\newtheorem{theorem}{Theorem}[section]

\usepackage{hyperref}
\usepackage{url}
\usepackage{cleveref}
\usepackage{graphicx}
\usepackage{wrapfig}
\usepackage{caption}

\title{Fast, Convex and Conditioned Network for Multi-Fidelity Vectors and Stiff Univariate Differential Equations}

\author{Siddharth Rout\\
Institute of Applied Mathematics\\
University of British Columbia\\
Vancouver, BC, Canada \\
\texttt{siddharth.rout@ubc.ca}\\
}

\begin{document}
\maketitle
\begin{abstract}
Accuracy in neural PDE solvers often breaks down not because of limited expressivity, but due to poor optimisation caused by ill-conditioning—especially in multi-fidelity and stiff problems. We study this issue in Physics-Informed Extreme Learning Machines (PIELMs), a convex variant of neural PDE solvers, and show that asymptotic components in governing equations can produce highly ill-conditioned activation matrices, severely limiting convergence. We introduce Shifted Gaussian Encoding, a simple yet effective activation filtering step that increases matrix rank and expressivity while preserving convexity. Our method extends the solvable range of Peclet numbers in steady advection–diffusion equations by over two orders of magnitude, achieves up to six orders lower error on multi-frequency function learning, and fits high-fidelity image vectors more accurately and faster than deep networks with over a million parameters. This work highlights that conditioning, not depth, is often the bottleneck in scientific neural solvers—and that simple architectural changes can unlock substantial gains. 
\end{abstract}


\section{Introduction}
\label{introduction}

Neural networks have demonstrated remarkable versatility and effectiveness across a broad range of applications, including images and audios \cite{8632885,9451544,peebles2023scalable}, natural language processing \cite{peebles2023scalable,chang2024survey,thirunavukarasu2023large}, and complex scientific modelling \cite{thiyagalingam2022scientific,cuomo2022scientific, pathak2022fourcastnet}. 
Despite these advancements, these powerful networks fail to impress the scientific community. For an instance, some simple mathematical or physical symbolic functions cannot be learned by deep neural networks \cite{liu2024kan}. Key issues arise while fitting to intricate data patterns, chaotic dynamics \cite{chattopadhyay2020data}, addressing differential equations (DEs) with varying orders of terms \cite{rout2019numerical,arzani2023theory}, ensuring stable training \cite{Goodfellow-et-al-2016}, and overcoming optimisation stiffness \cite{rout2019numerical}. When it comes to accurate prediction for large-scale scientific applications such as weather prediction and particularly for prediction of rare events such as heat waves, droughts, tornadoes, and cyclones, the inherent complex and chaotic high-dimensional nonlinear dynamics become very difficult to learn from. 

\subsection{Motivation}
\label{motivation}

\subsubsection{Neural differential equation solvers}
\label{ndesolver}
Solving differential equations is fundamental to many scientific and engineering fields, underpinning the modelling of phenomena in fluid dynamics, structural analysis, climate science, and finance. Traditional numerical methods—such as finite element, finite difference, and spectral methods—are well-established for their robustness and accuracy. However, they can be computationally intensive and often struggle with high-dimensional or complex problems. The emergence of automatic differentiation \cite{baydin2018automatic} and high-performance computing frameworks \cite{cuda2008, tensorflow2015, torch2017} has enabled deep learning \cite{lecun2015deep} to offer alternative approaches. These methods harness the function approximation capabilities of neural networks to address some of the limitations of classical techniques.

Physics-Informed Neural Networks (PINNs)—based on techniques developed decades ago \cite{dissanayake1994neural, lagaris1998artificial, RAISSI2019686}—have recently emerged as a powerful framework for solving differential equations. By embedding physical laws directly into the neural network's loss function, PINNs try to ensure that their predictions are consistent with the governing equations of the system. Unlike conventional machine learning methods, PINNs do not require training data and eliminate the need for specialised numerical schemes, meshing, or discretisation. This effectively recasts the problem of solving partial differential equations (PDEs) as a data-free optimisation task, where the solution is represented by a neural network that approximates the true function. Essentially, PINNs pose solving PDEs as optimisation problems. Hence, the exactness of the solution is dependent on how easy the optimisation problem is, and hence, proof of convergence is not guaranteed.  

The use of deep neural networks in this context offers several key advantages. Their universal approximation capability allows them to represent complex, nonlinear functions with high accuracy \cite{Cybenko1989ApproximationBS, HORNIK1989359, HORNIK1991251}. Once trained, these networks can produce solutions in real time, significantly reducing computational costs compared to traditional iterative solvers. Moreover, their flexibility enables them to handle noisy inputs and complex boundary conditions more effectively. Unlike traditional numerical methods, which solve discretised versions of the original equations and whose accuracy depends heavily on the number of collocation points, PINNs solve the exact differential equation. As a result, the accuracy of the solution is less dependent on the density of collocation points, especially when the target function is smooth.

PINNs have demonstrated their utility across various domains, including biomedics \cite{arzani2021, Sahli2020foPhys}, fluid mechanics \cite{raissi2020, cai2021physics}, uncertainty quantification \cite{Zhu_2019, YANG2019136}, and many more. Most of these successful applications are inverse problems. The method for forward problems is a bit tricky. The generalizability of the method is not questionable. However, when it comes to competing with the robustness of traditional methods, PINNs are lagging far behind \cite{femVspinn}. Rout \cite{rout2019numerical} shows that though neural networks are boasted for their strong and universal approximation capabilities which strengthen with increasing depth and width of a network, a small and shallow network could fit to the data points denoting the solution of an extremely advection-dominated ODE, but could not train itself to the loss function generating exact gradients from the governing ODE to update the weights. The depth and width of the network had a minimal role in the convergence of the loss function. Training PINNs is in fact an optimisation problem and hence, can be challenging, especially for applied problems which often involve sharp gradients, multiple constraints, and high frequencies. Rout \cite{rout2019numerical} emphasizes on advection-dominated PDEs while highlighting some of these issues and demonstrating ways to tackle them. They essentially adopt various optimisation tricks to reduce stiffness and optimise better. Some of the tricks, like piecewise approximation, weighted loss terms, and step/gradual optimisation, can be seen as effective. Stiff optimisation is a special characteristic evident in PINNs, unlike typical deep learning inverse problems. Wang et al. \cite{wang2021understanding} also show that the use of weighted loss terms is very effective in reducing the stiffness of the problem. Also for the same reason, most of the work using PINNs for forward problems \cite{RAISSI2019686, femVspinn, SHUKLA2021110683, jagtap2020extended, Shin2020OnTC} prefer a second-order optimiser specifically, L-BFGS. McClenny and Braga-Neto \cite{mcclenny2023self} use an attention mechanism to relax the stiffness. As the field of scientific machine learning continues to evolve, addressing these challenges will be crucial for further advancing the capabilities and applications of neural networks in solving complex PDEs across diverse scientific and engineering disciplines.

\subsubsection{Convexity of Networks}
\label{convexity}
Neural networks can be very non-convex. LeCun et al. \cite{lecun2002efficient} discussed the series of problems they faced while training a network, practically arising due to the non-convexity. Networks can come in many shapes and sizes, for which the use of higher-order optimisation methods may not look reasonable. This sparked interest among the computing community to explore convex surrogates or relaxations. Bengio et al. \cite{greedyLayerWise} characterised the optimisation problem for NNs as a convex program if the output loss function is convex in the NN output and if the output layer weights are regularised by a convex penalty. Bach \cite{bach2008consistency} studied connections between neural networks and convex optimisation by relating infinitely wide single-layer networks to kernel methods, which are convex in nature. Some later breakthroughs are primarily based on the ideas of shallow network and/or kernel method. Hazan et al. \cite{hazan2015beyond} demonstrated that certain classes of shallow neural networks could be trained using convex optimisation frameworks in online settings, motivating new algorithms. The recent breakthroughs have been linking convex optimisation tools to network learning. Jacot et al. \cite{jacot2018neural} proposed the neural tangent kernel, where the gradients of weights are mapped to those of the kernel method, which is known to be convex. Bach \cite{bach2017breaking} discussed how neural networks can exploit structured convex sets to simplify high-dimensional problems, further linking convex optimisation to modern deep learning. While Ergen and Pilanci \cite{ergen2021revealing} connected two-layer neural networks to convex optimisation through convex duality. Essentially, the aim has been to bridge neural networks with classic convex optimisation tools. 

In terms of approximation strength, a point to note is that the approximation ability of a single-layer network is also universal \cite{Cybenko1989ApproximationBS, bach2008consistency}. It seems sensible to simplify the network to understand the formulation of the cost function and not always think about a deep network. If the number of weights is big enough to provide a sufficient number of superpositions, as per \cite{Cybenko1989ApproximationBS}, then we are done. And, quite often, the required number of weights to express a function or set is much less than what we use, yet still the network does not fit. Rout \cite{rout2019numerical} in chapter 2 provides an example of such a case with some interesting results, which is also discussed in \cref{steadyAdvectionDiffusion}. The exact solution to the ODE problem mentioned there can be represented by a single-layered network with only two nodes. Essentially, we know the values the weights should converge to. However, on optimising the network, it can be noticed that it is not that easy to get the exact weights, especially in cases where advection dominates the diffusion more. Rout \cite{rout2019numerical} in chapter 3 demonstrated how if the loss function is a differential expression of the network output, the optimisation can get even tougher. In both the problem setups, the network is supposed to have learned the same function. With the advent of automatic differentiation \cite{baydin2018automatic}, many algorithms have gained popularity under the name umbrella name physics-informed neural networks \cite{cuomo2022scientific}. So, it is just to state that training a neural network can pose cumbersome optimisation problems because of how we pose the problem. Specifically, the training task should not always be about what the expressivity of a network is but also rather about how well the loss function can be optimised. A significant number of serious reasons are mentioned in the subsection \cref{orderOfscales} to push us to study the optimisation aspect of PDE problems.       

For the reason described in the previous section, we choose a particular kind of neural networks called extreme learning machines (ELM) \cite{huang2006extreme}. The speciality of this algorithm is its simplicity. To convert a network to ELM, we just fix all the weights and biases except those in the last layer. With this assumption, the network output becomes linear in terms of weights and biases. Hence, the training of this algorithm is solving a system of linear equations, which can be easily solved in a single step by Moore-Penrose generalised inverse \cite{rao1972generalized}. The good part is the convexity of this formulation, which is proved in \cref{th: elmconvexity}.

\subsubsection{Ill-conditioning of Networks}
\label{illcondition}
A network architecture can be called better than others if it can learn the task more, as in the approximation is closer to the exact solution, and/or it can learn faster. Training neural networks is a game of function approximation capability and optimisation. Neural networks, by default, are non-convex formulations \cite{ergen2021revealing}. Various studies have highlighted various obstacles and ill-conditionings present in the training process of deep networks \cite{Goodfellow-et-al-2016}. Hence, several attempts have been made over the decades to address such issues. Saarinen et al. \cite{cybenko1993illconditioning} showed that feedforward neural networks can easily have ill-conditioned Hessians and concluded that many network training problems are ill-conditioned and may not be solved more efficiently by higher-order optimisation methods. Smagt and Hirzinger \cite{vanderSmagt1998} showed that the ill-conditioning is because of the structure of the network, hence presented a less popular modified structure of the network. In comparatively recent advancements along the line, popular have been by various normalizations procedures \cite{batchNorm,ba2016layer,ulyanov2016instance}, like batch norm, layer norm and instance norm, to allow inputs to the active region of activation functions in each layer.  Some typical options to tackle such problems are the use of higher-order optimisers, regularisations, adaptive hyperparameterization, etc. Li et al. \cite{li2016preconditioned} relied upon gradient descent powered by Langevin dynamics to act as a stochastic preconditioner. Ill-conditioning can severely impact the efficiency and effectiveness of training, leading to slow convergence and sub-optimal performance. No certain way to completely eradicate ill-conditioning in optimisation problems exists. Hence, newer methods for conditioning are always encouraged and looked up to. 

\section{Contributions}
\label{contribution}
With these motivations, from optimisation and conditioning point of view, it makes sense to look into the training of ELM based differential equation solvers \cite{dwivedi2020physics,rout2019numerical}, called PIELM. The fundamental problem of possible ill-conditioning of such a solver is pointed out. Specifically, it is highlighted that asymptotic terms in an equation can be a reason for ill-conditioning. This issue can lead to significant stiffness in the training process, complicating the optimisation of neural models. Based on the understanding a novel neural architecture is proposed which gives well conditioned activation matrix. We propose a novel encoding/filtering procedure using shifted Gaussian functions and shifted ReLU to filter the activations from the previous layer in a particular fashion to generate a well-conditioned activation matrix.

\section{Preliminaries}
Let us say for any differential equation defined by \begin{equation}
\label{eq:1}
\mathcal{N}[u](\mathbf{x}) = 0,
\end{equation}
with constraints at $\mathbf{X}_u$ we have 
\begin{equation}
\label{eq:2}
u(\mathbf{X}_u) = \mathbf{u}_o.
\end{equation} 
The algorithms for solving the PDE defined by \cref{eq:1} and \cref{eq:2} using PINN and PIELM are mentioned in the subsequent subsections.

\subsection{Physics-Informed Neural Networks}
The \cref{pinn_algo} shows the pseudocode for PINNs as per the vanilla settings in \cite{RAISSI2019686}. The method uses training data points $(\mathbf{X}_u, \mathbf{u})$ to fit the neural network with trainable parameters $\theta$ denoted by $u_\theta(\mathbf{x})$, and minimize the residual of $\mathcal{N}[u](\mathbf{x})$ at collocation points $\mathbf{X}_f$.


\subsection{Physics-Informed Extreme Learning Machine (PI-ELM) Algorithm}

The PI-ELM algorithm integrates the principles of extreme learning machines (ELMs) with PINNs to solve differential equations efficiently. The vanilla ELMs can be considered as simplified neural networks since they typically have only a single layer of weights to learn and all other layers are fixed with random constants\cite{huang2006extreme}. Amazingly, despite such simplifications, the theory of universal approximation remains valid\cite{Huang2006UniversalAU, HUANG20073056}. The beauty of this method, which makes it special, is that the loss function could be represented as a set of equations that are nonlinear in terms of input variables and, however, linear in terms of trainable weights. Hence, the weights can be easily calculated by solving the system of linear equations, which is also the reason for the extreme speed of training single-layer networks.  

The \cref{algo_pielm} shows the pseudocode for PI-ELM. Similar to the problem definition in PINNs, this method also uses training data points $(\mathbf{X}_u, \mathbf{u})$ to fit the neural network with trainable parameters $\beta$ denoted by $u_\beta(\mathbf{x})$ and minimise the residual of $\mathcal{N}[u](\mathbf{x})$ at collocation points $\mathbf{X}_f$. Trained output weights $\mathbf{\beta}$ are obtained by solving the system of linear equations.

\begin{theorem}
\label{th: elmconvexity}
Let \(\mathcal{L}\) be a loss function of a single-layered neural network as per the definition of an extreme learning machine, whose trainable weights are \(\beta\), then \(\mathcal{L}\) is convex in \(\beta\).
\end{theorem}

\begin{proof}
Check \cref{pf: elmconvexity}
\end{proof}

\section{Proposed Architecture}
The architecture we propose is a layer of encoding after the layer of linear combination of weights to the inputs. Following the procedure of algorithm \cref{algo_pielm}, we make a change to the equation \cref{eq:elm} following the same set of definitions and notations. The replacement equation is stated as:
\begin{equation}
\label{encoded_ELM}
u_\beta(\mathbf{x}) = \phi(\mathbf{W} \mathbf{x} \odot E(\mathbf{x} - \mu) + \mathbf{b}) \mathbf{\beta},
\end{equation}
where $E$ is the encoding function, $\odot$ is a Hadamard product, $\mu$ is a set {$\mu_i = \frac{i}{L} | i \in [0, 1,..,L]$} and $L$ is one less than the number of hidden nodes.

\subsection{Shifted Gaussian Encoding} 
In this type, the encoding function $E$ is defined as:
\begin{equation}
E(x) = e^{-\frac{x^2}{d}},    
\end{equation}
where d is called the filter width.


\section{Experiments and Observations}
\label{expts}
\subsection{Scales of Gradients and 5 possible setbacks}
\label{orderOfscales}
Let us assume an ordinary differential equation as \cref{eq:normalODE}
\begin{equation}
\label{eq:normalODE}
a_0 u + a_1 \frac{du}{dx} + a_2 \frac{d^2 u}{dx^2} + .. + a_n \frac{d^n u}{dx^n} = b.
\end{equation}

Like the concept we have been discussing, in order to solve the above equation by minimising the L2 norm of residuals by fitting neural networks $N(\mathbf{w},x)$ at random collocation points in the domain, we define the loss function $L$ as: 
\begin{equation}
\label{eq:normalLoss}
L(\mathbf{w},x) = || a_0 N(\mathbf{w},x) + a_1 \frac{dN(\mathbf{w},x)}{dx} + .. + a_n \frac{d^n N(\mathbf{w},x)}{dx^n} - b ||_2.
\end{equation}

Since we opt for gradient descent to find the least square fit, we can update each trainable weight ($\mathbf{w}$) by the following update \cref{eq:gradDesc}.

\begin{equation}
\label{eq:gradDesc}
\mathbf{w}_{k+1} = \mathbf{w}_k - \eta \nabla L(\mathbf{x}_k, x)
\end{equation}
where $\nabla$ is the gradient of the loss function and $\eta$ is the learning rate. Where, 
\begin{equation}
\label{eq:grad}
\nabla L(\mathbf{x}_k, x) = \sqrt{L}
(a_0 \frac{\partial N}{\partial \mathbf{w}} + a_1 \frac{\partial^2 N}{\partial \mathbf{w} \partial x} + .. + a_n \frac{\partial^{n+1} N} {\partial \mathbf{w} \partial x^n} ).
\end{equation}

Looking at \cref{eq:grad}, we can make a couple of serious observations. Firstly, the gradient descent depends on a set of derivatives of different orders of the network output and in nonlinear cases, it can depend on derivatives of different degrees as well. This should be leading to competing gradients due to each term suggesting a clear difficulty in finding a gradient step for reaching the global or Pareto optimum of the residual. Secondly, more is the number of different terms with different orders, the tougher it gets for the optimiser as it has to deal with more terms of different orders of magnitudes. It might require multiscale tactics. Thirdly, the approximating test function has to be flexible and capable enough to learn a function whose derivatives of different orders and are of different scales. The more varying the scales of each order of differentials are, which is a definite thing in perturbation or asymptotic problems, the more robust the test function might be needed. Fourthly, for the same conceptual reason, with reference to \cref{eq: multiobj}, the loss function has two kinds of loss terms as described in \cref{eq: datafit} for data fitting for the variable and \cref{eq: pdefit} for fitting to the differential equations. As such, this is a multi-objective optimisation problem, which brings in the chances of non-Pareto solutions. Of course when the scales of the data fitting error and the PDE residuals do not match. It demands opting for weighted residual instead of averaged residuals as in vanilla PINN. The same observation is described in a couple of works \cite{ rout2019numerical, wang2021understanding}, where the relative scaling of terms and variables helps us in getting better solutions.  Additionally, from the perspective of machine precision and numerical computing, at perturbed or asymptotic regions in the domain, some of the derivatives can jump off the machine limit and hence crash the training or poorly approximate crucial gradients. Knowing all these factors, now when we checkout the work by Grossman et al. \cite{femVspinn} on why neural differential equation solvers cannot beat finite element methods, we would not be surprised. It would not be wrong to say that using neural differential equation solvers is primarily solving optimisation problems.

\begin{wrapfigure}{l}{0.55\textwidth}
    \centering
    \includegraphics[width=0.55\textwidth]{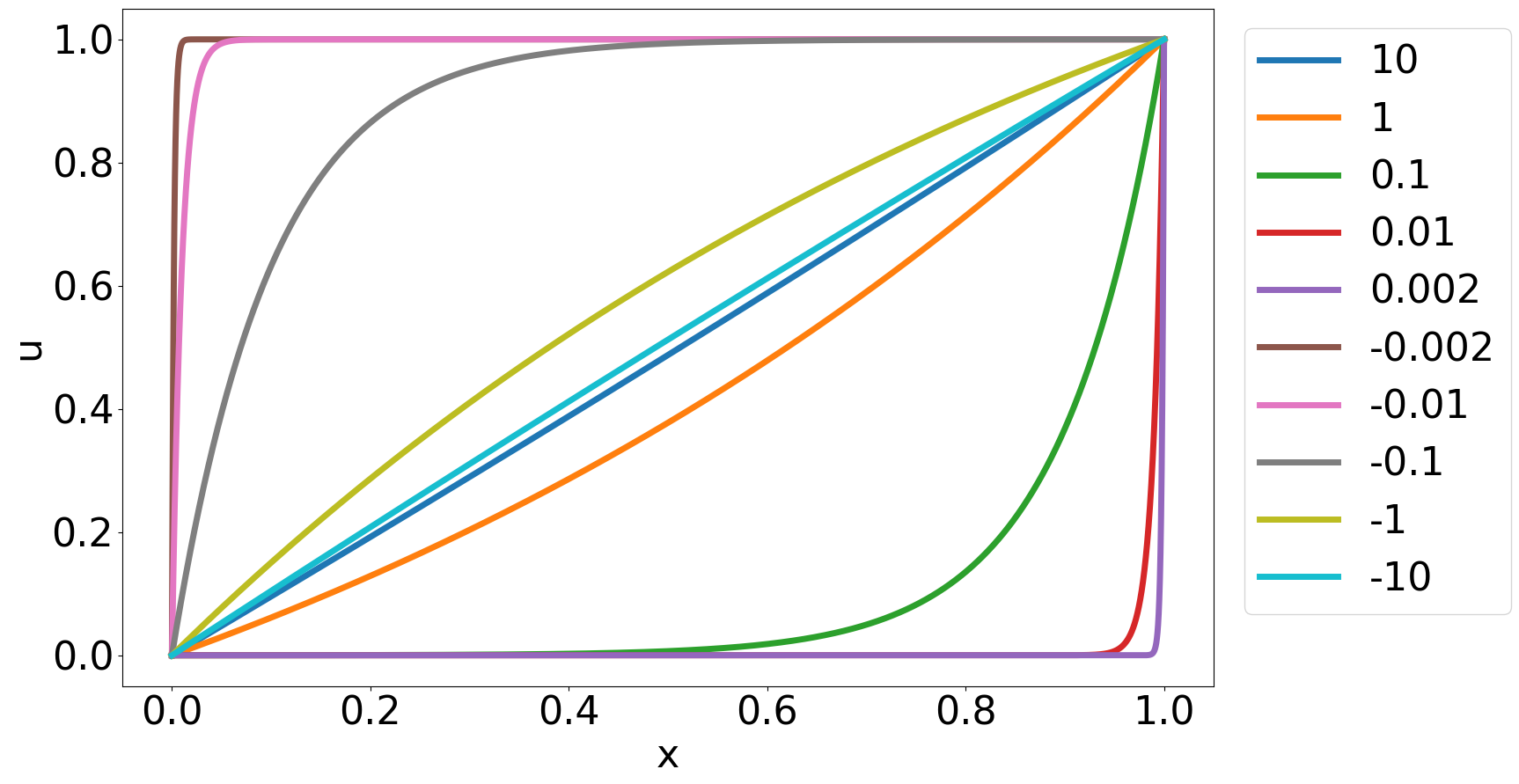}
    \caption{Exact solutions of steady 1D ADE for different $\epsilon$.}
    \label{fig_1dAD}%
\end{wrapfigure}

\subsection{Solving Steady 1-Dimensional Linear Stiff Advection Diffusion Equation}
\label{steadyAdvectionDiffusion}

Let us consider the case of solving one-dimensional steady advection diffusion equation (ADE) for $u$ where $\epsilon \in \mathbb{R}$ is the diffusivity (a constant) and $x \in [0,1]$ as in \cref{eq:1dADE}
\begin{equation}
\label{eq:1dADE}
\frac{\partial u}{\partial x} = \epsilon \frac{\partial^2 u}{\partial x^2},
\end{equation}
\begin{equation}
u(x=0) = 0, \text{ and } u(x=1) = 1.\\
\end{equation}

Luckily, we know the exact solution of this problem. We can derive it by integrating the equation twice and substituting the boundary values. It is given by

\begin{equation}
\label{eq:1dADE_exact}
u(x) = \frac{e^{x/\epsilon} - 1}{e^{1/\epsilon} - 1}.
\end{equation}

The good part is that the solution is a function of a single exponential activation function, which means the function is expressible with very few nodes (2+). The \cref{fig_1dAD} shows the solution for different diffusivities. The solution is very sensitive to perturbations around zero diffusivity and gives a very different solution for the left and the right limits. The smaller is the diffusivity, the sharper is the gradient. Numerically, solving this problem for smaller diffusivity is tougher, as it requires much finer discretisation, at least around the region of sharper gradients. A realistic scale of diffusivity can be much smaller than what we see in the figure, typically found in sonic waves and fluid dynamics.    

The \cref{fig_1dAD_PIELM} shows that a typical neural differential equation solver (PIELM \cite{rout2019numerical}) can solve for $\epsilon \in (\approx 0.06, \infty) $. The finite element approach of piecewise approximation with neural networks, works better, but is expensive and looses the charm of using a powerful approximator. The \cref{fig_1dAD_PIDELM} shows how such a solver, PIDELM \cite{dwivedi2020physics}, can solve for $\epsilon \in (\approx 0.02, \infty) $. The \cref{fig_1dAD_PINDELM} shows how such a solver with asymptotic scaling based transformation, PINDELM \cite{rout2019numerical}, can solve for $\epsilon \in (\approx 0.0025, \infty)$.

So, the motif of exploring how well we can solve using a single neural network for a side ODE remains unfulfilled. This motivates us to study the functional composition of the network and the loss function. 

\subsubsection{Statistics of Activation Matrix}
We start looking at the steps in the algorithm of PIELM, refered in \cref{algo_pielm}. The formulation of the loss function is convex as proved in \cref{th: elmconvexity}. So, ill-conditioning of inner function might be an obvious possibility. \cref{eq:PIELM_Inversion} is the step we focus on. In the equation, we look at the statistics of the activation matrix $H$.

For a case, where we take an ELM of 1000 nodes to solve for $\epsilon = 1.0$ using 1000 collocation points, which generates a square activation matrix $H$, of 1002 rows, and hence the maximum possible rank is 1002. The case is solvable with appreciable accuracy. However, on generating the matrix, the rank it has is 14, the determinant is $\approx 0.0$, and the condition number is $1.11e21$. The rank is too small and the condition number is too large. The \cref{fig_PIELM_eigens} shows how poor the condition is. The \cref{fig_PIELM_A} shows how dense the matrix is. It is clear evidence that the linear system we are solving is extremely ill-conditioned. Accordingly, the \cref{fig_PIELM_WnB} shows how, depending on $\epsilon$, the magnitudes of residual (mean absolute error) as well as the weights increase sharply as it reduces down from $\approx 0.1$. This suggests why a penalty loss term using L1 or L2 regularisation on weights, might help. 

With PIELM on our architecture with shifted Gaussian encoding, we look upon the statistics of the activation matrix $H$. For a case, where we take an ELM of 1000 nodes to solve for $\epsilon=1.0$ using 1000 collocation points, which generates a square activation matrix $H$, of 1002 rows, and hence the maximum possible rank is 1002. The case is solvable with appreciable accuracy. On generating the activation matrix with filter width, d, is 0.0001, the rank gets to 702, the determinant is $\approx 0.0$, and the condition number is $7.68e20$. Though the condition number has not improved much, the rank has improved significantly. The \cref{fig_PIELM_eigens_our} shows better information in the activation matrix. The \cref{fig_PIELM_A_our} shows how the activation matrix is now sparse and diagonal.

\begin{table}[!h]
\centering 
\begin{tabular}{l c c c} 
 \hline
 Architecture & $Domain(\epsilon)$ & Rank($H$) & $\mathcal{O}(MAE)$ \\
 & & & for $\epsilon = 0.01$ \\ 
 \hline
 Typical Network &  [$\approx$1e-1, $\infty$) & 14 & 1e-1 \\
 \textbf{Our Network} & \textbf{[$\approx$1e-3, $\infty$)} & \textbf{702} & \textbf{1e-8} 	 \\ 
 \hline
\end{tabular}
\vspace{0.2cm}
\caption{Comparison of results from using a typical single-layered neural network and our network with a layer of filtering. $\mathcal{O}(MAE)$ is the order-of-magnitude mean absolute error}
\label{TableComparion}
\end{table}

\begin{figure}
        \begin{minipage}[h]{.48\textwidth}
	\centering 
	\includegraphics[width=\textwidth]{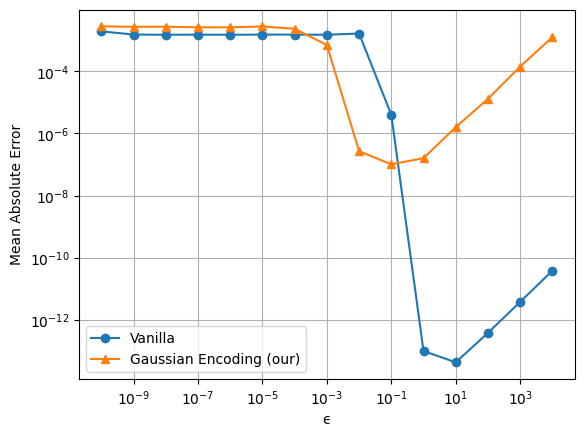}	
	\caption{Comparison of mean absolute error for different $\epsilon$.}  
	\label{fig_comp}
    \end{minipage}
    \hfill
    \begin{minipage}[h]{.48\textwidth}
	\centering 
	\includegraphics[width=\textwidth]{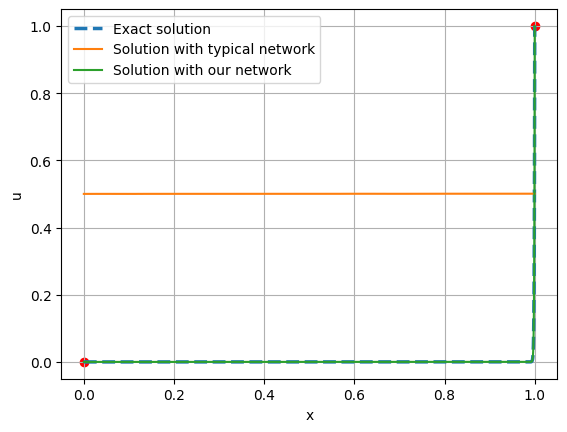}	
	\caption{Solution to an extremely small diffusivity $\epsilon = 1e-3$ with red dots depicting the exact boundaries.}  
	\label{fig_1e-6}
    \end{minipage}
\end{figure}

\subsection{Learning Multi-fidelity Trends in Data}
\label{stiff Data}
The goal of this study is to learn stiff and complex patterns in data. So, here the aim is designed as to train a neural network to approximate a complex target function that exhibits non-linear, oscillatory, and rapidly varying behaviour using the data points generated by the function. The function is defined by:
\begin{equation}
f(x) = \sin(10x) + 0.2 \cos(20x + 50x^2) + e^{-100x^2} \sin(200x),
\end{equation}

where:
\begin{itemize}
    \item The term \(\sin(10x)\) introduces a smooth oscillatory component with a moderate frequency.
    \item The term \(0.2 \cos(20x + 50x^2)\) adds a high-frequency oscillatory component modulated by a quadratic phase term.
    \item The term \(e^{-100x^2} \sin(200x)\) contributes a highly localised oscillatory feature, with rapid damping due to the Gaussian factor \(e^{-100x^2}\).
\end{itemize}

The generated data points can be seen in the \cref{fig_stiffdata}. The problem we propose here is not that easy. Learning this function poses several challenges that make this problem very exciting. The presence of oscillatory terms with varying high frequencies, such as \(\sin(10x)\) and \(\sin(200x)\), demands that the network accurately captures and represents these stiff variations. The Gaussian term \(e^{-100x^2}\) introduces sharp, localised features that require the network to effectively resolve fine details in the input space. The quadratic phase term \(50x^2\) in \(\cos(20x + 50x^2)\) creates a non-linear dependency of frequency that complicates the learning task. 

\begin{figure}[!h]
	\centering 
	\includegraphics[width=\linewidth]{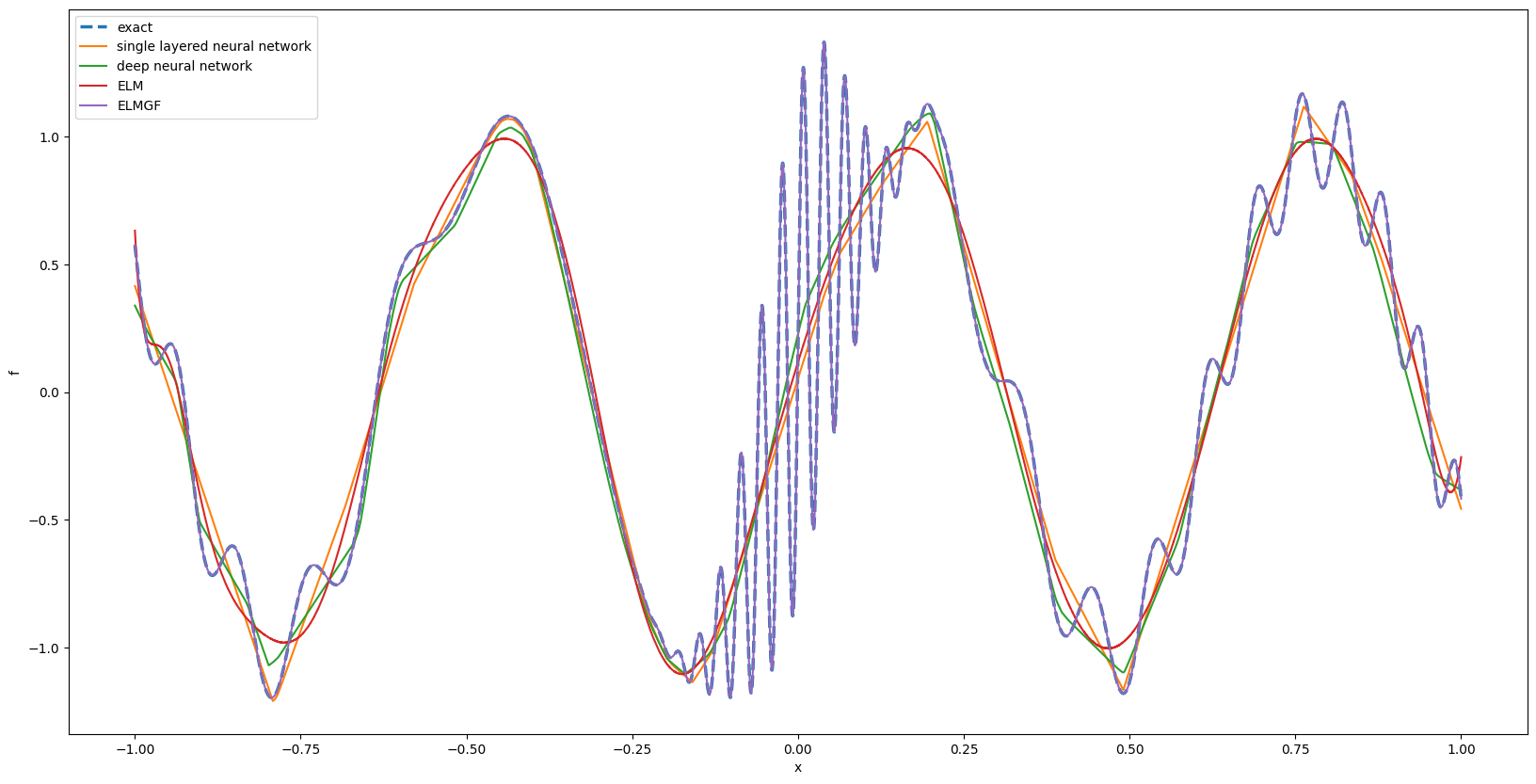}	
	\label{fig_stiffdatacomp}
        \caption{Trends learned with various methods.} 
\end{figure}

The results to compare the learnability with different methods can be observed in \cref{fig_stiffdata}. A deep neural network (DNN) with four hidden layers with 100, 1000, 1000, and 100 nodes, respectively, with ReLU activation is used to demonstrate that even a deep network with more than a million parameters is not able to learn the high-frequency patterns. After that, as per reducing approximation strength, a single-layered network (SLN) with 1000 nodes and an ELM with 1000 nodes are used, and as expected, the mean squared error (MSE) increases as per the sequence of approximation strength. The \cref{TableDataComparion} shows the error values for different methods. It is clearly evident that our ELM with Gaussian Encoding, which is conditioned for better expressivity of activations having the same 1000 hidden nodes, has a much smaller MSE. This suggests that it is not the approximation capability we need to look for, but how well the model can be trained is an important factor. Also our conditioned ELM takes a lesser runtime to find the pseudoinverse and hence learns the weights faster. It can also be observed in \cref{TableDataComparion}. The Gaussian filter width (d) in our model controls the condition of the activation matrix. So, with a smaller d, we make the activation matrix of ELM get sparser and a better (higher) rank. So, if d is 1e-2 it shows a little better approximation than typical methods, while if d is 1e-3, it shows much more better results as it learns almost all visible trends.    

An ELM with the first set of weights randomly assigned and the activation matrix we see has a rank of 16 out of a maximum of 1000. The \cref{fig_ELMMatrix} shows the density distribution of the gram activation matrix of a typical ELM with our data points from the complex functional. The ELM takes in 10000 distinct data points to express the information with 1000 hidden nodes, that is, our activation matrix of 10000 rows and 1000 columns. Its rank is 16, and the network could have expressed 1000 different eigenvectors, while we do not recover the information  16 out of 16 might be a low-level expression of the available information in the data points. whereas when the activation matrix of the same size in the case of our method is compared, it can express the information from the same 10000 data points through 1000 hidden nodes as 356 different eigenvectors, that is, the rank we get. Our model is in fact able to learn the last set of weights easily. Our model could perform better as it could not only learn dominant low frequency eigenmodes but also a sufficient number of imapctful high-frequency eigenmodes. Our gram matrix can be seen in \cref{fig_OurELMMatrix}, it is much more structured. So, a great learning is to not just take a big network but also understand if the activation from the previous layer of a network is expressing the information to the next layer sufficiently. Ideally, we would wish to create a layer that sufficiently expresses itself. In case the input data is expressible with 1000 vectors and our model layer has 100 nodes for a lower-order representation so, the rank of an ideal activation matrix should be 100. If not we should try to keep the rank as high as possible. Having 100 nodes yet the rank is let's say 10, would mean only 10-20 nodes would have been sufficient to perform what it is performing at the moment.  

\begin{table}
\centering 
\begin{tabular}{l c c c} 
 \hline
 Method & \# Trainable & MSE & Training Time \\
 & Parameters &  & (s) \\
 \hline
 DNN & 1.2M+ & 4.2e-3 & 179.968 \\
 SLN & 3001 & 4.5e-3 & 31.448 \\
 ELM & 1000 & 4.5e-3 & 6.686 	 \\ 
 \textbf{Our ELM (d=1e-2)} & \textbf{1000} & \textbf{1.1e-3} & \textbf{5.548} \\
 \textbf{Our ELM (d=1e-3)} & \textbf{1000} & \textbf{4.9e-09} & \textbf{4.694} \\
 \hline
\end{tabular}
\vspace{0.2cm}
\caption{Comparison of results from using a typical single-layered neural network and our network with a layer of filtering.}
\label{TableDataComparion}
\end{table}

\subsection{Fitting VanGogh's Eye}
\label{Fit_Img}
A single-channel image of Van Gogh's eye is vectorised to pose a complex multiscale vector to be fitted. The \cref{TableImgComparion} shows the comparison of performance and \cref{fig_vangogh} shows the images obtained by different methods.

\begin{table}
\centering 
\begin{tabular}{l c c c} 
 \hline
 Method & \# Trainable & MSE & Training Time \\
 & Parameters &  & (s) \\
 \hline
 DNN & 1.2M+ & 1.39e-2 & 118.67 \\
 SLN & 3001 & 1.41e-2 & 5.35 \\
 ELM & 1000 & 1.36e-2 & 11.76 	 \\ 
 \textbf{Our ELM (d=1e-2): ELMGF} & \textbf{1000} & \textbf{1.32e-2} & \textbf{18.62} \\
 \textbf{Our ELM (d=1e-4): ELMGF2} & \textbf{1000} & \textbf{3.29e-3} & \textbf{7.22} \\
 \textbf{Our ELM (d=1e-5): ELMGF3} & \textbf{5000} & \textbf{1.23e-3} & \textbf{220.87} \\
 \hline
\end{tabular}
\vspace{0.2cm}
\caption{Comparison of results for fitting the image of Van Gogh's Eye.}
\label{TableImgComparion}
\end{table}

\begin{figure}[!h]
	\centering 
	\includegraphics[width=\linewidth]{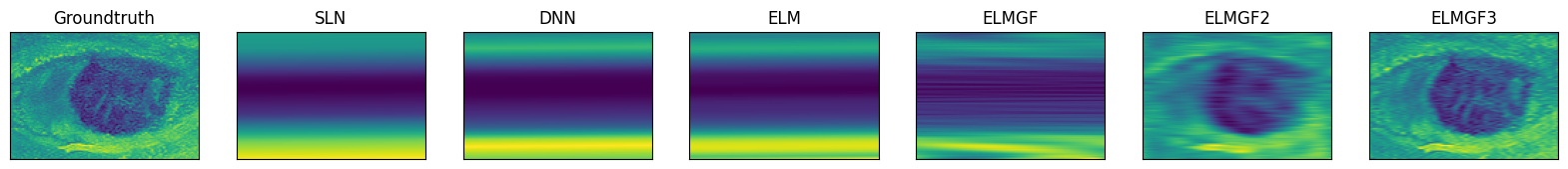}	
	\caption{Learned images with various methods.} 
	\label{fig_vangogh}
    \vspace{-0.5cm}
\end{figure}

\section{Results and Discussions}
Consequently, \cref{fig_comp} shows how, depending on $\epsilon$, the magnitudes of the residual (mean absolute error) are different for different architectures. Shifted Gaussian encoding improves the result from vanilla architecture over a comparatively wider domain. The results with shifted Gaussian encoding are great achievements for computational scientists working with hyperbolic PDEs, as no neural network could ever approximate such a sharp gradient in advection-diffusion problems. The other best solutions for singularly perturbed problems either involve high fidelity higher order discretised finite approximation that are expensive and might require manual modelling \cite{Stynes_2005,MOHEBBI20103071}, otherwise solve for asymptotic approximations \cite{arzani2023theory, KADALBAJOO20103641} which poorly match the whole domain and fail at different scales \cite{Hinch_1991}. Similarly, the method is brilliant at fitting to complex one-dimensional equations and vectors.

\section{Conclusion}
A conditioning-focused architecture for scientific neural solvers by combining the convexity of Extreme Learning Machines with \textit{Shifted Gaussian Encoding} is presented. This approach substantially improves the rank and sparsity of activation matrices, enabling PIELMs to solve stiff DEs and capture multi-scale patterns far beyond the reach of conventional shallow or deep baselines.

The experiments demonstrate broader solvability, including advection--diffusion equations down to $\epsilon \approx 10^{-3}$ without domain decomposition or asymptotic preprocessing; higher accuracy at lower cost, achieving orders-of-magnitude improvement in error for oscillatory functions while reducing pseudoinverse computation time; and strong generality, with robust performance across synthetic functions, DEs, and high-dimensional vectors such as image fitting.

These results suggest that, for many scientific learning problems, improving \emph{conditioning} can be more impactful than increasing network depth or width. Future work will focus on formalising the link between encoding width and matrix spectrum, and exploring adaptive encodings for dynamic conditioning across training.

\section*{Acknowledgements}
Special thanks to Prof. Eldad Haber, Prof. Michael Friedlander, Prof. Balaji Srinivasan, Dr. Vikas Dwivedi, and Benjamin Liu.

\bibliography{references}
\bibliographystyle{abbrv}

\clearpage

\appendix
\section{Theorems and Proofs}
\begin{theorem}
Let \(\mathcal{L}\) be a loss function of a single-layered neural network as per the definition of an extreme learning machine whose trainable weights are \(\beta\), then \(\mathcal{L}\) is convex in \(\beta\).
\end{theorem}

\begin{proof}
\label{pf: elmconvexity}
Given, \(\mathcal{L}\) is the loss function of an extreme learning machine where let the single-layered neural network be (\(\mathcal{N}\)) of \(L\) hidden nodes. Let (\(\mathcal{N}\)) be a function of input \(x \in \mathbb{R}^{C_1 \times n}\), fixed weight \(W \in \mathbb{R}^{L \times C_1}\), fixed bias \(b \in \mathbb{R}^L\) and trainable weight \(\beta \in \mathbb{R}^{C_2 \times n}\)  to predict as defined by
\begin{equation}
\mathcal{N}(X, \beta, W, b) = \beta\phi(Wx + b),
\end{equation}
where \(\phi\) is a nonlinear activation function, \(L \in \mathbb{N}\) is the number of nodes in the hidden layer, \(C_1 \in \mathbb{N}\) is the number of input nodes, \(C_2 \in \mathbb{N}\) is the number of output nodes and \(n\) is the number of samples.

If the output is \(Y \in \mathbb{R}^{C_2 \times n}\), 
\begin{equation}
\mathcal{L}(\beta) = ||Y - \beta\phi(Wx + b)||_2.
\end{equation}

So, the gradient with respect to \(beta\) is given by
\begin{equation}
\nabla\mathcal{L}(\beta) = 2(Y - \beta\phi(Wx + b))(-\phi(Wx + b))^T.
\end{equation}

So, the Hessian with respect to \(beta\) is given by
\begin{equation}
\nabla^2\mathcal{L}(\beta) = 2(\phi(Wx + b))(\phi(Wx + b))^T.
\end{equation}

If \(A = \phi(Wx + b)\) then the Hessian \(\mathcal{H}\) can be written as,

\begin{equation}
\mathcal{H} = 2 AA^T.
\end{equation}

Since \(AA^T\) is positive semidefinite, \(AA^T \geq 0\).
So,
\begin{equation}
    \mathcal{H} = \nabla^2\mathcal{L}(\beta) \geq 0.
\end{equation}

Hence, \(\mathcal{L}\) is convex in terms of the weight \(\beta\).

\end{proof}
\section{Pseudocodes}
The \cref{pinn_algo} shows the algorithm for training physics-informed neural networks. The \cref{algo_pielm} shows the algorithm for training physics-informed extreme learning machine.
\begin{algorithm}[!h]
\caption{PINN}
\begin{algorithmic}[1]
\label{pinn_algo}

\STATE \textbf{Initialization}
Initialize neural network parameters $\theta$ (weights and biases)

\STATE \textbf{Define the Neural Network}
Construct the neural network $u_\theta(\mathbf{x})$ with parameters $\theta$

\STATE \textbf{Formulate the Loss Functions}
Define data fitting loss:
\begin{equation}
\label{eq: datafit}
\mathcal{L}_{\text{data}} = \frac{1}{N_u} \sum_{i=1}^{N_u} \left| u_\theta(\mathbf{x}_u^i) - u^i \right|^2.
\end{equation}

Define physics-informed loss using the governing differential equation $\mathcal{N}[u] = 0$:
\begin{equation}
\label{eq: pdefit}
\mathcal{L}_{\text{phys}} = \frac{1}{N_f} \sum_{j=1}^{N_f} \left| \mathcal{N}[u_\theta](\mathbf{x}_f^j) \right|^2.
\end{equation}
Combine the losses:
\begin{equation}
\label{eq: multiobj}
\mathcal{L} = \mathcal{L}_{\text{data}} + \mathcal{L}_{\text{phys}}.
\end{equation}

\STATE \textbf{Training}
Minimize the total loss $\mathcal{L}$ with respect to the network parameters $\theta$:
\begin{equation}
\theta^* = \arg\min_\theta \mathcal{L}(\theta).
\end{equation}

\STATE \textbf{Inference}
Use the trained network $u_{\theta^*}$ to make predictions for new inputs $\mathbf{X}_{\text{new}}$:
\begin{equation}
\mathbf{\hat{u}} = u_{\theta^*}(\mathbf{X}_{\text{new}}).
\end{equation}

\end{algorithmic}
\end{algorithm}

\begin{algorithm}[!h]
\caption{PI-ELM}
\begin{algorithmic}[1]
\label{algo_pielm}

\STATE \textbf{Initialization}
Randomly initialize input weights $\mathbf{W} \in \mathbb{R}^{L \times n}$ and biases $\mathbf{b} \in \mathbb{R}^L$.
Define activation function $\phi(\cdot)$

\STATE \textbf{Formulate Composite Loss Function}
Define data fitting loss: 
\begin{equation}
\mathcal{L}_{\text{data}} = \sum_{i=1}^{N_u} \left| u(\mathbf{x}_u^i) - u^i \right|^2.
\end{equation}
Define physics-informed loss:
\begin{equation}
\mathcal{L}_{\text{phys}} = \sum_{i=1}^{N_f} \left| \mathcal{N}[u(\mathbf{x}_f^i)] \right|^2.
\end{equation}
Combine the losses:
\begin{equation}
\mathcal{L} = \mathcal{L}_{\text{data}} + \mathcal{L}_{\text{phys}}.
\end{equation}

\STATE \textbf{Generate the System of Equations}
The equation for the neural network $u_\beta(\mathbf{x})$ is
\begin{equation}
\label{eq:elm}
u_\beta(\mathbf{x}) = \phi(\mathbf{W} \mathbf{x} + \mathbf{b}) \mathbf{\beta}.
\end{equation}

The effective system of linear equations in $\mathcal{L}$ could be rewritten in the form:
\begin{equation}
\label{eq:PIELM_Inversion}
\mathbf{H} \mathbf{\beta}  = \mathbf{T}.
\end{equation}

\STATE \textbf{Training}
Solve for the output weights using the Moore-Penrose pseudoinverse of $\mathbf{H}$ denoted by $\mathbf{H}^+$ which is essentially the least square solution:
\begin{equation}
\mathbf{\beta} = \mathbf{H}^+ \mathbf{T}.
\end{equation}

\STATE \textbf{Inference}
Use the trained PI-ELM model to predict outputs for new inputs $\mathbf{X}_{\text{new}}$:
\begin{equation}
\mathbf{\hat{Y}} = \phi(\mathbf{W} \mathbf{X}_{\text{new}} + \mathbf{b}) \mathbf{\beta}.
\end{equation}

\end{algorithmic}
\end{algorithm}

\section{PINNs 1D Steady Advection-Diffusion Equation}
\subsection{Comparison of existing methods}
See \cref{fig_1dAD_PIELM,fig_1dAD_PIDELM,fig_1dAD_PINDELM}.
\begin{figure}[!h]
        \centering
        \includegraphics[width=\linewidth]{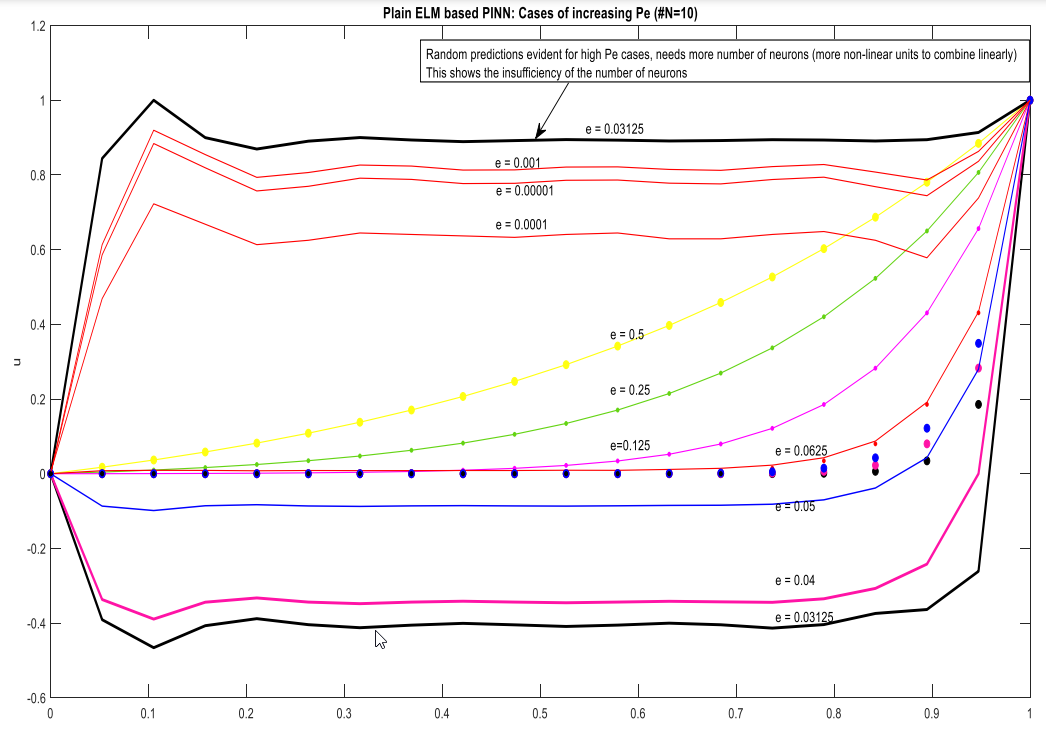}	
	\caption{Solutions of steady 1D ADE for different $\epsilon$ values with PIELM in continuous lines and exact solutions in dots.} 
        \label{fig_1dAD_PIELM}
\end{figure}

\begin{figure}[!h]
        \centering
        \includegraphics[width=\linewidth]{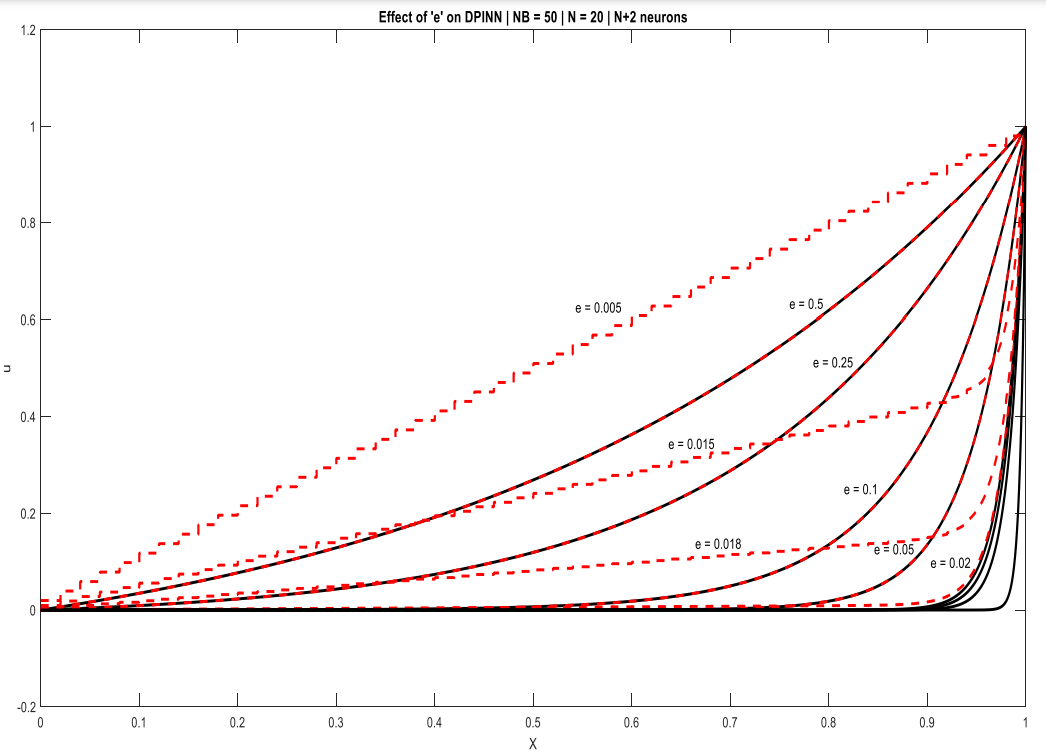}	
	\caption{Solutions of steady 1D ADE for different $\epsilon$ values with PIDELM in red dashed lines and exact solutions in black lines.} 
	\label{fig_1dAD_PIDELM}
\end{figure}

\begin{figure}[!h]
        \centering
        \includegraphics[width=\linewidth]{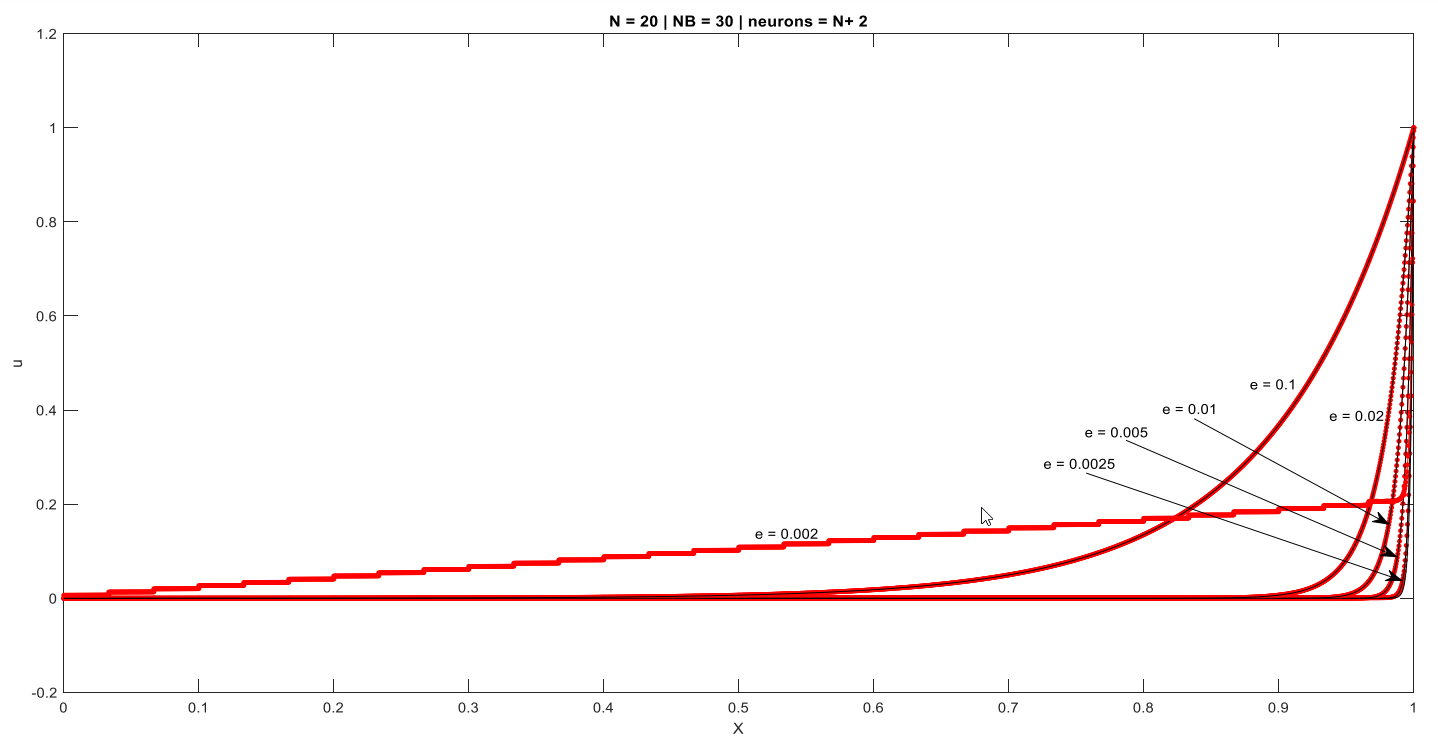}	
	\caption{Solutions of steady 1D ADE for different $\epsilon$ values with PINDELM in red dotted lines and exact solutions in black lines.} 
	\label{fig_1dAD_PINDELM}
\end{figure}

\subsection{Comparison of activation matrices}
See \cref{fig_PIELM_A,fig_PIELM_eigens,fig_PIELM_A_our,fig_PIELM_eigens_our}.
\begin{figure}[!h]
        \begin{minipage}[h]{.4\textwidth}
	\centering 
	\includegraphics[width=\textwidth]{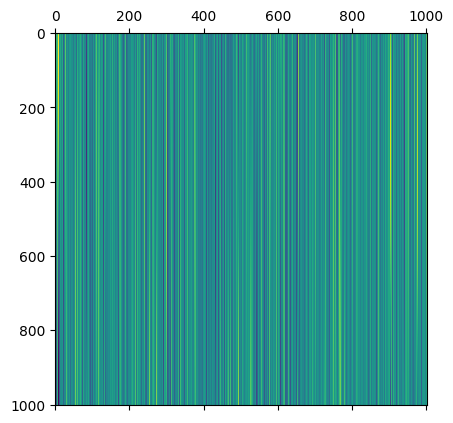}	
	\caption{Visualisation of $H$ of PIELM for steady 1D ADE with $\epsilon = 1.0$.} 
	\label{fig_PIELM_A}%
    \end{minipage}
    \hfill
    \begin{minipage}[h]{.5\textwidth}
	\centering 
	\includegraphics[width=\textwidth]{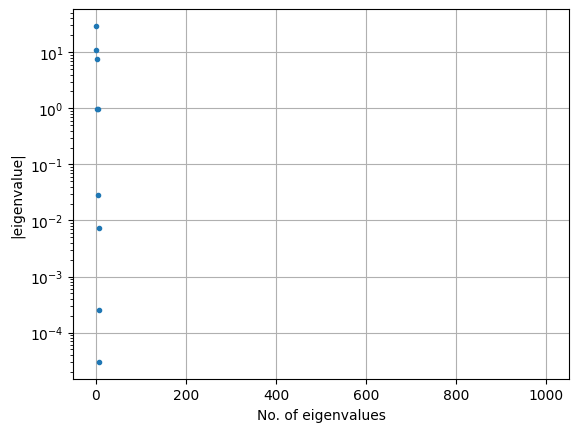}	
	\caption{Eigenvalues of $H$ of PIELM for steady 1D ADE with $\epsilon = 1.0$.} 
	\label{fig_PIELM_eigens}%
    \end{minipage}
    \vspace{-0.5cm}
\end{figure}

\begin{figure}[!h]
        \begin{minipage}[h]{.4\textwidth}
	\centering 
	\includegraphics[width=\textwidth]{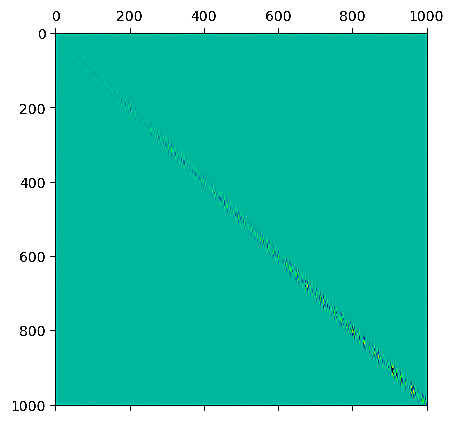}	
	\caption{Visualisation of $H$ of PIELM with our architecture for steady 1D ADE with $\epsilon = 1.0$ when shifted Gaussian encoding is used.} 
	\label{fig_PIELM_A_our}%
    \end{minipage}
    \hfill
    \begin{minipage}[h]{.5\textwidth}
	\centering 
	\includegraphics[width=\textwidth]{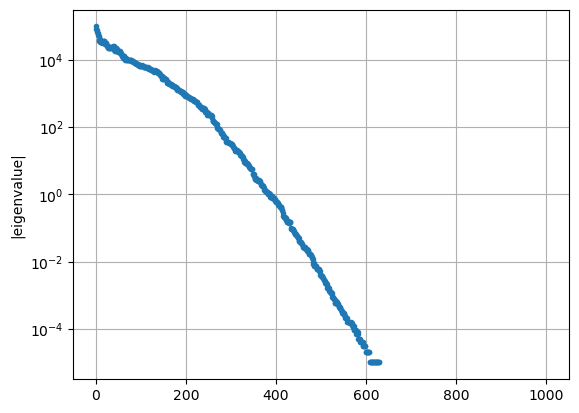}	
	\caption{Eigenvalues of $H$ of PIELM with our architecture for steady 1D ADE with $\epsilon = 1.0$ when shifted Gaussian encoding is used.} 
	\label{fig_PIELM_eigens_our}%
    \end{minipage}
    \vspace{-0.5cm}
\end{figure}

\subsection{Comparison of orders of weights and residuals}
See \cref{fig_PIELM_WnB,fig_PIELM_WnB_our}.
\begin{figure}[!h]
	\centering 
	\includegraphics[width=\textwidth]{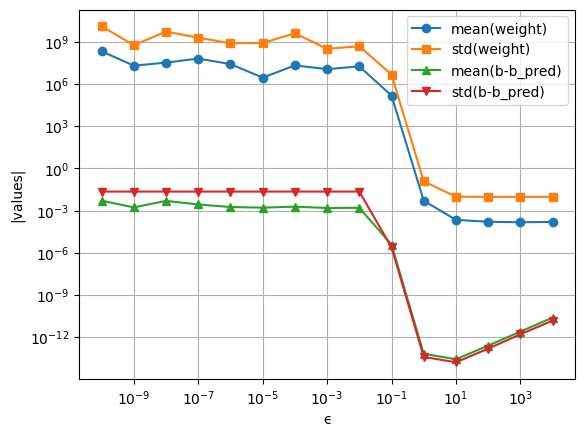}	
	\caption{Mean and standard deviation of weights and residual (b - b\_pred) of PIELM solution of steady 1D ADE for various $\epsilon$ values.} 
	\label{fig_PIELM_WnB}
\end{figure}

\begin{figure}[!h]
	\centering 
	\includegraphics[width=\textwidth]{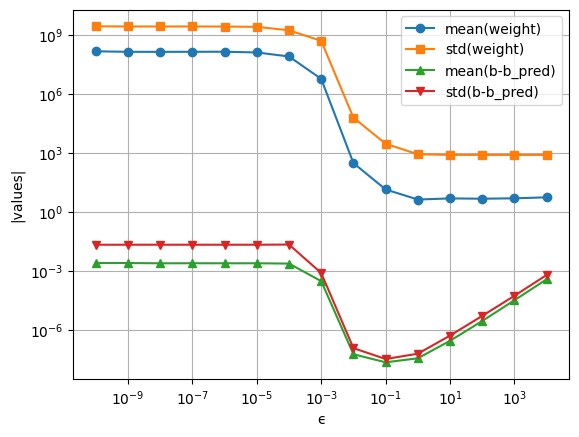}	
	\caption{Mean and standard deviation of weights and residual (b - b\_pred) of PIELM solution of steady 1D ADE for various $\epsilon$ values when shifted Gaussian encoding is used.} 
	\label{fig_PIELM_WnB_our}
\end{figure}

\section{Complex Multiscale Function}
See \cref{fig_stiffdata} for the plot of the exact multiscale function.
\begin{figure}[!h]
	\centering 
	\includegraphics[width=\linewidth]{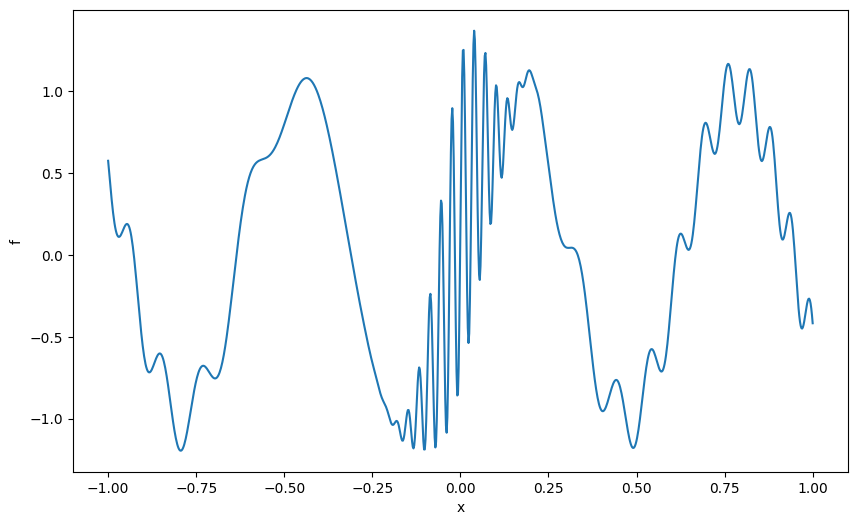}	
	\caption{Points generated from the target function.} 
	\label{fig_stiffdata}
\end{figure}
\subsection{Comparison of activation matrices}
See \cref{fig_ELMMatrix,fig_OurELMMatrix} for the comparison of Gram activation matrices formed in different approaches.
\begin{figure}[!h]
        \begin{minipage}[h]{.48\textwidth}
	\centering 
	\includegraphics[width=\textwidth]{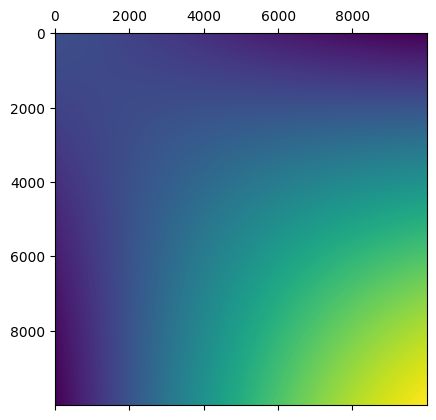}	
	\caption{Visualisation of Activation matrix for the ELM.} 
	\label{fig_ELMMatrix}
    \end{minipage}
    \hfill
    \begin{minipage}[h]{.48\textwidth}
	\centering 
	\includegraphics[width=\textwidth]{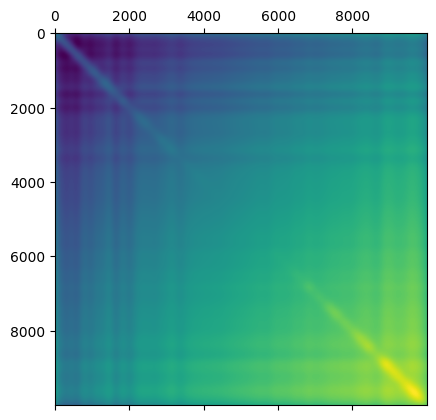}	
	\caption{Visualisation of Activation matrix for our ELM after Gaussian Encoding.} 
	\label{fig_OurELMMatrix}
    \end{minipage}
\end{figure}

\section{Multiscale Image}
\subsection{Comparison of activation Gram matrices}
See \cref{fig_vGMatrix} for comparison of the Gram matrix formed in various methods.
\begin{figure}[!h]
        \begin{minipage}[h]{.3\textwidth}
	\centering 
	\includegraphics[width=\textwidth]{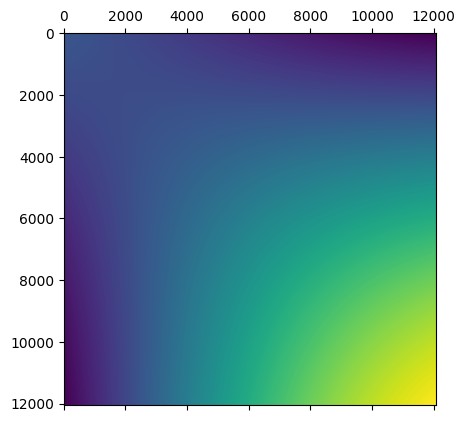}	
	\caption{Visualisation of Activation matrix for the ELM.}
    \end{minipage}
    \hfill
    \begin{minipage}[h]{.3\textwidth}
	\centering 
	\includegraphics[width=\textwidth]{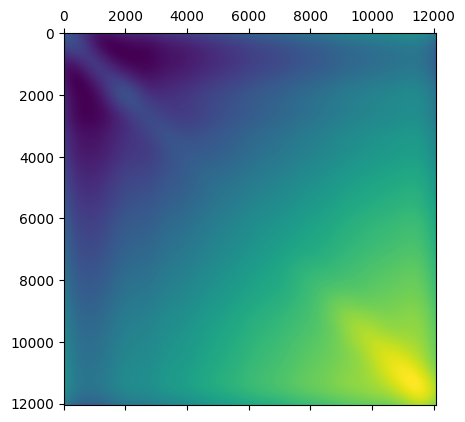}	
	\caption{Visualisation of Activation matrix for our ELM (ELMGF).} 
    \end{minipage}
    \hfill
    \begin{minipage}[h]{.3\textwidth}
	\centering 
	\includegraphics[width=\textwidth]{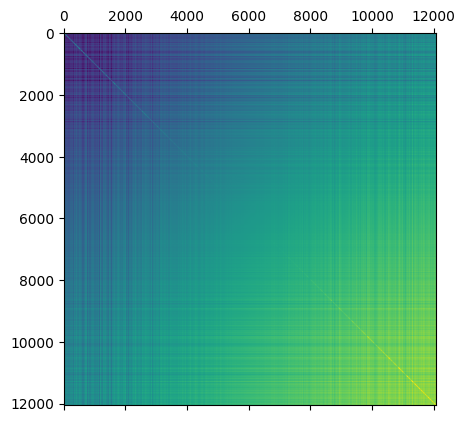}	
	\caption{Visualisation of Activation matrix for our ELM (ELMGF3).} 
    \end{minipage}
    \captionsetup{labelformat=empty}
    \caption{Fitting Van Gogh's Eye.}
    \label{fig_vGMatrix}
\end{figure}

\end{document}